\documentclass[twoside]{article}

\usepackage[accepted]{aistats2019}
\usepackage{amssymb,amsmath,amsthm}
\usepackage{xspace}
\usepackage{amsthm}
\usepackage{amsmath}
\usepackage{amssymb}
\usepackage{mathrsfs}

\usepackage{tikz}
\usetikzlibrary{positioning}
\usetikzlibrary{decorations.text}
\usetikzlibrary{arrows}
\usetikzlibrary{shapes}

\usepackage{dsfont}

\usepackage{algorithm,algorithmicx,algpseudocode}

\usepackage{amsmath,amssymb,mathtools,subfigure}

\usepackage{url}
\usepackage{xcolor}

\newcommand{\alice}[0]{\textit{Alice}}
\newcommand{\bob}[0]{\textit{Bob}}
\newcommand{\eve}[0]{\textit{Eve}}
\newcommand{\sm}[0]{\textit{sm}}
\newcommand{\fr}[0]{\textit{fr}}

\newtheorem{theorem}{Theorem}
\newtheorem{lemma}{Lemma}
\newtheorem{proposition}[theorem]{Proposition}
\newtheorem{definition}{Definition}
\newtheorem{remark}[theorem]{Remark}
\newtheorem{example}[theorem]{Example}

\newif\ifappendix
\appendixtrue

\newif\ifmaintext
\maintexttrue

\allowdisplaybreaks

\begin{document}

%

%

\twocolumn[

\ifmaintext
\aistatstitle{Lifted Weight Learning of Markov Logic Networks Revisited}
\else
\aistatstitle{Lifted Weight Learning of Markov Logic Networks Revisited (Appendix)}
\fi

\aistatsauthor{ Ond\v{r}ej Ku\v{z}elka \And Vyacheslav Kungurtsev }

\aistatsaddress{Czech Technical University in Prague \\ Dept of Computer Science, KU Leuven \And Czech Technical University in Prague} ]

\ifmaintext

\begin{abstract}
We study lifted weight learning of Markov logic networks. We show that there is an algorithm for maximum-likelihood learning of 2-variable Markov logic networks which runs in time polynomial in the domain size. Our results are based on existing lifted-inference algorithms and recent algorithmic results on computing maximum entropy distributions.
\end{abstract}

\section{INTRODUCTION}

Statistical Relational Learning \cite{getoor2007introduction} (SRL) is concerned with learning probabilistic models from relational data. Markov Logic Networks \cite{Richardson2006} (MLNs) are among the most prominent SRL systems. An MLN is given by a set of weighted first-order logic formulas and a domain $\Delta$. Generative weight learning of MLNs is typically performed using  maximum-likelihood estimation. Unfortunately, maximizing likelihood of MLNs is generally intractable. Therefore, in practice, one often resorts to heuristic approximations. Another option besides using approximations is to restrict the class of MLNs to those for which inference can be performed efficiently. This has been studied in the subarea of SRL called {\em lifted inference} \cite{braz2005lifted}. In particular, it has been shown in \cite{van2011lifted,van2014skolemization} that probabilistic inference in MLNs with formulas containing at most 2 logical variables can be performed in time polynomial in the size of the given domain $\Delta$. This has been exploited in \cite{van.haaren.mlj} for maximum-likelihood learning of MLNs, 
suggesting tractable learning of 2-variable MLNs could be possible. However, although it showed that gradients of log-likelihood can be computed efficiently, it did not provide a bound on the total runtime of the learning algorithm, specifically, because this bound was missing a guarantee on the number of iterations of the optimization algorithm.

In this paper, we complete the work of \cite{van.haaren.mlj} by answering whether maximum-likelihood learning of MLNs can be done in time polynomial in the size of the domain for 2-variable MLNs. We give a positive answer to this question (Theorem \ref{thm:main}), under consideration of the dependence of the runtime bounds on how extreme the statistics of the training data are. To arrive at this positive result, we need to combine results from three streams of research: (i) lifted inference \cite{van2011lifted,beame2015symmetric}, (ii) links between maximum-likelihood learning of MLNs and relational marginal problems \cite{kuzelka2018relational}, and (iii) algorithmic results on maximum-entropy distributions \cite{singh2014entropy}. 
We should note here that our results are mostly of theoretical interest. Making the algorithms described in this paper practical would be potential future research.

The rest of the paper is structured as follows. After covering the necessary background material in Section \ref{sec:background}, we introduce the concept of interiority in relational marginal polytopes in Section \ref{sec:polytopes}. We then state our main technical results in Section \ref{sec:main}. Then, in Sections \ref{sec:2var_polytopes}, \ref{sec:bounding_box}, we work towards the proof of the main results which we finish in Sections \ref{sec:proof} and \ref{sec:proof:corollary}. The paper is concluded in Section~\ref{sec:conclusions}.


\section{BACKGROUND}\label{sec:background}

\subsection{First-Order Logic}

We consider a function-free first-order logic language $\mathcal{L}$, built from a set of constants $\textit{Const}$, variables $\textit{Var}$ and predicates $\textit{Rel} = \bigcup_i \textit{Rel}_i$, where $\textit{Rel}_i$ contains the predicates of arity $i$. We assume an untyped language (all our results can be straightforwardly generalized to the typed case). For $a_1,...,a_k \in \textit{Const}\cup \textit{Var}$ and $R \in \textit{Rel}_k$, we call $R(a_1,...,a_k)$ an atom.  If $a_1,..,a_k\in \textit{Const}$, this atom is called ground. A literal is an atom or its negation. 
We use $\textit{Vars}(\alpha)$ to denote the variables that appear in a formula $\alpha$.
The formula $\alpha_0$ is called a grounding of $\alpha$ if $\alpha_0$ can be obtained by replacing each variable in $\alpha$ with a constant from $\textit{Const}$. 
A formula is called closed if all variables are bound by a quantifier. A variable in a formula is called free if it is not bound by a quantifier. A formula with no free variables is called a sentence. A formula is called quantifier-free if all variables in it are free.
A possible world $\omega$ is defined as a set of ground atoms. 
A substitution is a mapping from variables to terms. An injective substitution is a substitution which does not map any two variables to the same variable or constant. 


\subsection{Markov Logic Networks}\label{sec:mlns}


A Markov logic network \cite{Richardson2006} (MLN) is a set of weighted first-order logic formulas $(\alpha,w)$, where $w\in \mathbb{R}$ and $\alpha$ is a function-free and quantifier-free first-order formula. The semantics are defined w.r.t.\ the groundings of the first-order formulas, relative to some finite set of constants $\Delta$, called the domain. An MLN is classically seen as a template that defines a Markov random field (in Section \ref{sec:rel_marg_problems}, we describe another way of interpreting MLNs--as solutions to max-entropy marginal problems). Specifically, an MLN $\Phi$ induces the following probability distribution on the set of possible worlds $\omega \in \Omega$: 
$
p_{\Phi}(\omega) = \frac{1}{Z} \exp \left(\sum_{(\alpha,w) \in \Phi} w \cdot N(\alpha,\omega)\right),
$
where $N(\alpha, \omega)$ is the number of injective\footnote{Normally, MLNs are not defined with injective groundings. However, working with injective groundings turns out to be more convenient and equally expressive \cite{kuzelka2018relational,DBLP:conf/aaai/BuchmanP15}.} groundings of $\alpha$ satisfied in $\omega$, and $Z$ is a normalization constant to ensure that $p_{\Phi}$ is a  probability distribution.

\subsection{Ellipsoid Algorithm}\label{sec:ellipsoid}

In this section we briefly describe the main properties of the ellipsoid algorithm for convex optimization \cite{boyd2004convex}; the exposition is based on \cite{singh2014entropy}. Consider an arbitrary convex optimization problem,
\[
\begin{array}{rl}
\max_{\lambda\in\mathbb{R}^m} & g(\lambda) \\
\text{s.t.} & h_i(\lambda)=0,\,\forall i\in\{1,...,k\}
\end{array}
\]
where $g$ is concave and $h_i$ are all affine. Assume that $g$ and $h_i$ are differentiable
everywhere, and furthermore, there exists a \emph{strong first order oracle} for $g$ which, given
$\lambda$, outputs $g(\lambda)$ and $\nabla g(\lambda)$ and that we can project $\nabla g(\lambda)$
onto the affine space defined by $K=\{\lambda: h_i(\lambda)=0,\,\forall i\in\{1,...,k\}\}$.

The ellipsoid algorithm will be used in the proofs in this paper as it satisfies the following property,
\begin{theorem}\cite[Theorem 2.13]{singh2014entropy}\label{th:ellipsoid}
Given any $\beta,R>0$, there exists an algorithm, namely the \emph{ellipsoid algorithm} which,
given a strong first-order oracle for $g$, returns a $\widehat\lambda$ such that,
\begin{multline*}
    g(\widehat\lambda) \ge \max_{\lambda\in K,\|\lambda\|_\infty\le R} g(\lambda) \\\qquad+
\beta\left(\min_{\lambda\in K,\|\lambda\|_\infty\le R} g(\lambda)-\max_{\lambda\in K,\|\lambda\|_\infty\le R} g(\lambda)\right)
\end{multline*}
and the number of calls to the strong first-order oracle is bounded by a polynomial in $m$, $\log R$
and $\log(1/\beta)$.
\end{theorem}

\subsection{Relational Marginal Problems}\label{sec:rel_marg_problems}

In this section we describe the relationship between MLN weight learning using maximum likelihood estimation and so-called {\em relational marginal problems} which were studied in \cite{kuzelka2018relational}. 

We start by defining {\em formula statistics} which are closely related to {\em random-substitution semantics} \cite{bacchus_halpern_koller,schulte}. In our case, the formula statistics are just rescaled numbers of true groundings of a formula (defined in Section \ref{sec:mlns}), where the scaling depends on the number of variables in the formula.

\begin{definition}[Formula statistics]
Let $\alpha$ be a quantifier-free first-order logic formula with $k$ variables $\{x_1,\dots,x_k\}$. We define its formula statistic w.r.t.\ a possible world $\omega$ as:
$$Q_\omega(\alpha) = \left( \begin{array}{c} |\Delta| \\ k \end{array} \right)^{-1} \cdot (k!)^{-1} \cdot N(\alpha, \omega).$$ 
\end{definition}

\begin{remark}
When $\alpha$ does not contain any variables, e.g.\ when $\alpha = \textit{smokes}(\textit{Alice})$, then $Q_\omega(\alpha) \in \{ 0,1\}$.
\end{remark}

Intuitively, for a given formula $\alpha$ and a possible world $\omega$, the formula statistic $Q_\omega(\alpha)$ is the probability that the ground formula $\alpha\vartheta$ is true where $\vartheta$ is a grounding injective substitution of $\alpha$'s free variables picked from all such substitutions uniformly at random.\footnote{This is how formula statistics relate to random substitution semantics \cite{bacchus_halpern_koller,schulte}.}

\begin{example}\label{ex1}
Let $\omega = \{ \fr(\alice,\bob),$ $\fr(\bob,\alice),$ $\fr(\bob,\allowbreak\eve),$ $\fr(\eve,\bob),$ $\sm(\alice) \}$ and $\Delta = \{\alice, \bob, \eve \}$, i.e.\ the only smoker is $\alice$ and the friendship structure is:
\resizebox{0.215\textwidth}{!}{
\tikzset{
main node/.style={ellipse,fill=white!11,draw,minimum size=0.3cm,inner sep=0pt},
other node/.style={rectangle,fill=white!11,minimum size=0.3cm,inner sep=0pt},
}
\tikzset{edge/.style = {->,> = latex'}}
\begin{tikzpicture}

\node[main node] (1) {Alice};
\node[main node] (2) [right = 0.5cm of 1] {Bob};
\node[main node] (3) [right = 0.5cm of 2] {Eve};

\draw[edge] (1) [bend right] to (2);
\draw[edge] (2) [bend right] to (1);
\draw[edge] (3) [bend right] to (2);
\draw[edge] (2) [bend right] to (3);
\end{tikzpicture}}
Let 
$\alpha = \fr(x,y) \Rightarrow \sm(y).$
We then get $Q_{\omega}(\alpha) = \frac{1}{2}$ (of the 6 possible injective substitutions $\vartheta$ of $\alpha$'s variables, three make $\alpha\vartheta$ true in $\omega$).
\end{example}

\begin{remark}
Let us have a set $\Omega$ of possible worlds over a domain $\Delta$. MLNs over $\Omega$, given by a set of weighted formulas $\Phi = \{(\alpha_1,w_1), \dots, (\alpha_l,w_l) \}$, can be re-defined using formula statistics as:
$$p_\Phi(\omega) = \frac{1}{Z} \exp{\left( \sum_{(\alpha_i,w_i) \in \Phi} w_i \cdot Q_\omega(\alpha_i) \right)}.$$
For possible worlds over a domain of fixed size, the only difference is the scaling factor in the definition of formula statistics, which is fixed for each formula and fixed domain size, hence, as a result the only difference is that the weights need to be scaled as well. In what follows when we refer to MLNs we will mean this representation unless stated otherwise.
\end{remark}


Next we use formula statistics to define a maximum entropy distribution over a given set of possible worlds $\Omega$. Assuming that we know the values $\theta_1, \dots, \theta_l$ that the formula statistics of the given formulas $\alpha_1, \dots, \alpha_l$ should have in expectation (which we might have, for instance, estimated from given training data), we can define the following convex optimization problem encoding the maximum entropy problem.

\paragraph{Relational Marginal Problem (Formulation):}

\begin{align}
    \min_{\{ P_\omega \colon \omega \in \Omega \}}  \sum_{\omega\in \Omega} P_{\omega} \log{{P_\omega}} \quad \textit{ s.t.}\label{eq:maximum_entropy_criterion} \\
    \forall i = 1,\dots,l: \sum_{\omega \in \Omega} P_\omega \cdot Q_\omega(\alpha_i) = \theta_i\label{eq:maxent_marginal_constraints} \\
\forall \omega \in \Omega : P_{\omega} \geq 0, \sum_{\omega \in \Omega} P_{\Omega} = 1\label{eq:maximum_entropy_cnormalization}
\end{align}

\noindent Here, $P_\omega$'s are the decision variables of the problem, each representing probability of one possible world $\omega \in \Omega$. The first line (\ref{eq:maximum_entropy_criterion}) is the maximum entropy criterion (represented here as minimization of negative entropy), (\ref{eq:maxent_marginal_constraints}) are constraints given by the statistics and (\ref{eq:maximum_entropy_cnormalization}) are normalization constraints for the probability distribution.

Assuming there exists a feasible solution satisfying $\forall \omega : P_\omega > 0$, the optimal solution of the above maximum entropy problem is an MLN 
\begin{equation}
    P_\omega = p_{\Phi}(\omega) = \frac{1}{Z} \exp{\left( \sum_{(\alpha_i,\lambda_i) \in \Phi} \lambda_i \cdot Q_\omega(\alpha_i) \right)} \label{eq:mln-as-sol}
\end{equation}
where the parameters $\mathbf{\lambda} = (\lambda_1,\dots,\lambda_l)$ are obtained by maximizing the dual criterion
\begin{equation}
  L(\mathbf{\lambda}) = \sum_{\alpha_i} \lambda_i \theta_i - \log{\sum_{\omega \in \Omega} e^{\sum_{\alpha_i} \lambda_i Q_\omega(\alpha_i)}}\label{eq:dual}  
\end{equation}
This dual criterion also happens to be equivalent to the log-likelihood of the MLN (\ref{eq:mln-as-sol}) w.r.t.\ a (possibly fictitious) training example $\widehat{\omega}$ that has to be over the same domain $\Delta$ and that satisfies $Q_{\widehat{\omega}}(\alpha_i) = \theta_i$ for all the formula statistics.

\begin{remark}
Due to the above duality, if we can show that we can solve relational marginal problems efficiently, it will follow as a corollary that we can solve maximum likelihood estimation in MLNs efficiently and vice versa. 
\end{remark}

\begin{remark}
Above, we have used the assumption that there exists a feasible solution where probability of every possible world is positive. This does not hurt generality of our discussion because we can always remove the possible worlds $\omega$ that, by the virtue of the given constraints, must have zero probability in any feasible solution from the set $\Omega$. In most cases, $\Omega$ is not given explicitly but by means of a first-order logic theory (that describes which possible worlds are ``possible''), so it is enough to add suitable first-order sentences to this theory.
\end{remark}

\subsection{Inference Using Weighted Model Counting}\label{sec:wfomc}

To maximize the dual criterion (\ref{eq:dual}) we will need to be able to compute its gradient. For the partial derivatives of (\ref{eq:dual}), we have
\begin{multline}
    \frac{\partial L}{\partial \lambda_i} = \theta_i - \frac{\sum_{\omega \in \Omega} Q_\omega(\alpha_i) \cdot e^{\sum_{\alpha_i} \lambda_i Q_\omega(\alpha_i)}}{\sum_{\omega \in \Omega} e^{\sum_{\alpha_j} \lambda_j Q_\omega(\alpha_j)}} \\
    = \theta_i - \mathbb{E}[Q_\omega(\alpha_i)] \label{eq:gradient}
\end{multline}
It follows that, in order to compute the gradient, we will also need to be able to compute the partition function $Z = \sum_{\omega \in \Omega} e^{\sum_{\alpha_j} \lambda_j Q_\omega(\alpha_j)}$. 
Computation of the partition function $Z$ can be converted to a {\em first-order weighted model counting problem (WFOMC)}.

\begin{definition}[WFOMC \cite{van2011lifted}]
Let $w(P)$ and $\overline{w}(P)$ be functions from predicates to real numbers (we call $w$ and $\overline{w}$ {\em weight functions}) and let $\Phi$ be a first-order theory. Then $\operatorname{WFOMC}(\Phi,w,\overline{w}) =$
$$
     = \sum_{\omega \in \Omega : \omega \models \Phi} \prod_{a \in \mathcal{P}(\omega)} w(\textit{Pred}(a)) \prod_{a \in \mathcal{N}(\omega)} \overline{w}(\textit{Pred}(a))
$$
where $\mathcal{P}(\omega)$ and $\mathcal{N}(\omega)$ denote the positive literals that are true and false in $\omega$, respectively, and $\textit{Pred}(a)$ denotes the predicate of $a$ (e.g. $\textit{Pred}(\textit{friends}(\textit{Alice},\textit{Bob})) = \textit{friends}$).
\end{definition}

To compute the partition function $Z$ using weighted model counting, we may proceed as in \cite{van2011lifted}. Let a set of weighted formulas $\Phi$ be given. Here, for simplicity of exposition, we will assume that the formulas in $\Phi$ do not contain constants (we refer to \cite{van2011lifted} for the general case). For every weighted formula $(\alpha_i,\lambda_i) \in \Phi$, where the free variables in $\alpha_i$ are exactly $x_1$, $\dots$, $x_k$, we create a new formula 
\begin{multline*}
    \forall x_1,\dots,x_k : \xi_i(x_1,\dots,x_k) \Leftrightarrow (\alpha_i(x_1,\dots,x_k) \wedge  \\
    x_1 \neq x_2 \wedge x_1 \neq x_3 \wedge \dots \wedge x_{k-1} \neq x_k)
\end{multline*}
where $\xi$ is a new fresh predicate. Then we set
$$w(\xi_i) = \exp{\left( \left( \begin{array}{c} |\Delta| \\ |\textit{Vars}(\alpha_i)| \end{array} \right)^{-1} \cdot (|\textit{Vars}(\alpha_i)|!)^{-1} \cdot \lambda_i \right)}$$ 
and $\overline{w}(\xi_i) = 1$ and for all other predicates we set both $w$ and $\overline{w}$ equal to 1. It is easy to check that then $WFOMC(\Phi,w,\overline{w}) = Z$, which is what we needed to compute. To compute the numerator of (\ref{eq:gradient}), we need to compute $WFOMC(\Phi \cup \{ \alpha_i \vartheta \},w,\overline{w})$ where $\vartheta$ is an injective grounding substitution of $\alpha_i$.


Importantly, there are classes of first-order logic theories for which weighted model counting is polynomial-time. In particular, as shown in \cite{van2014skolemization}, when the theory consists only of first-order logic sentences, each of which contains at most two logic variables, the weighted model count can be computed in time polynomial\footnote{Here, we should note that the runtime of these WFOMC algorithms depends on the parameters of the theory $\Phi$ exponentially. However, in many cases, these parameters are small compared to size of the domain.} in the number of elements in the domain $\Delta$ over which the set of possible worlds $\Omega$ is defined. This is not the case in general when the number of variables in the formulas is greater than two unless P = \#P$_1$~\cite{beame2015symmetric}.

\begin{remark}
It has already been shown in \cite{van.haaren.mlj} that gradients of log-likelihood of an MLN can be computed efficiently whenever WFOMC can be computed efficiently (in fact, the translation described in this section for computing $Z$ is essentially the same as the one described in \cite{van.haaren.mlj}). 
\end{remark}



\section{MARGINAL POLYTOPES}\label{sec:polytopes}

Not all possible values of formula statistics correspond to actual probability distributions. 

\begin{example}\label{ex:densities}
Let $\alpha = e(x_1,x_2)$, $\beta = e(x_1,x_2) \wedge e(x_2,x_3) \wedge e(x_3,x_1)$ and let $\Delta = \{ c_1, \dots, c_{100} \}$ be the set of domain elements and $\Omega$ be the respective set of possible worlds over the first-order language given by the predicate $e/2$ and the constants from $\Delta$. We can think of possible worlds $\omega \in \Omega$ as directed graphs (the predicate $e/2$ representing edges in the graph and the constants in $\Delta$ representing vertices). Then $Q_\omega(\alpha)$ corresponds to ``density'' of edges and $Q_\omega(\beta)$ to ``density'' of directed triangles. It is then easy to see why there is, for instance, no distribution with $\mathbb{E}[Q_\omega(\alpha)] = 0$ and $\mathbb{E}[Q_\omega(\beta)] = 0.5$ (since graphs with no edges obviously cannot have positive density of triangles).
\end{example}

The points corresponding to values of statistics that correspond to some actual probability distributions form what is called a {\em relational marginal polytope} \cite{kuzelka2018relational}.

\begin{definition}[Relational marginal polytope]
Let $\Omega$ be a set of possible worlds and $\Phi = (\alpha_1, \dots, \alpha_l)$ be a list of formulas. We define the relational marginal polytope $\textit{RMP}(\Phi,\Omega)$ w.r.t. $\Phi$ as 
\begin{multline*}
    \textit{RMP}(\Phi,\Omega) = \{ (x_1,\dots,x_l) \in R^l : \exists \mbox{ prob. distr. on } \\
     \Omega\; { s.t. }\;\mathbb{E}[Q_\omega(\alpha_1)] = x_1 \wedge \dots \wedge \mathbb{E}[Q_\omega(\alpha_l)] = x_l \}.
\end{multline*}
\end{definition}

\begin{remark}It is not difficult to see that the relational marginal polytope w.r.t.\ a given list of formulas $(\alpha_1$, $\dots$, $\alpha_l)$ can be equivalently defined as the convex hull of the set 
$\{ (Q_\omega(\alpha_1),\dots,Q_\omega(\alpha_l)) : \omega \in \Omega \}$.
\end{remark}

Next we define what it means for a point to be in the $\eta$-interior of a polytope.

\begin{definition}[Interiority]
Let $\eta > 0$, $\mathbf{P}$ be a polytope and $A^= \mathbf{x} = \mathbf{c}$ be the maximal linearly independent system of linear equations that hold for the vertices of $\mathbf{P}$. A point $\theta$ is said to be in the $\eta$-interior of $\mathbf{P}$ if 
$\{\theta' | A^= \theta' = \mathbf{c}, \| \theta' - \theta \| \leq \eta \} \subseteq \mathbf{P}. $
\end{definition}

\noindent The reason why we need to consider the system of linear equations $A^= \mathbf{x} = \mathbf{c}$ in the definition of interiority is because it may happen that the polytope lives in a lower dimensional subset of the given space. We note that interiority, as we defined it, is also often called {\em relative interiority} in the literature.

\begin{remark}
When we were constructing the dual relational marginal problem, we had to assume that there is a positive solution that satisfies the constraints of the primal problem. It is not difficult to see that if the vector of formula statistics' estimates $\theta$ is in the $\eta$-interior of the respective relational marginal polytope for some $\eta > 0$ then such a solution always exists. To see this, first, notice that an interior point $\theta$ can be repesented as a convex combination 
$\theta = \sum_{\mathbf{x} \in \{ (Q_\omega(\alpha_1),\dots,Q_\omega(\alpha_l)) : \omega \in \Omega \}} a_\mathbf{x} \cdot \mathbf{x}$
where $a_\mathbf{x} > 0$ for all $\mathbf{x} \in \{ (Q_\omega(\alpha_1),\dots,Q_\omega(\alpha_l)) : \omega \in \Omega \}$. To find a positive distribution over $\Omega$ that satisfies the constraints, we just need to assign positive probabilities $P_\omega$ so that 
$a_\mathbf{x} = \sum_{\omega \in \Omega : (Q_\omega(\alpha_1),\dots,Q_\omega(\alpha_l)) = \mathbf{x}} P_\omega$, which we can always do.
\end{remark}

\section{MAIN RESULTS}\label{sec:main}

In this section we describe our main technical result which is showing that maximum-likelihood weight learning of 2-variable MLNs can be done in time polynomial in the size of the domain (i.e.\ the problem is domain-liftable \cite{van2011lifted}). As already mentioned in the previous sections, it has been shown that computing log-likelihood and its derivatives is domain liftable \cite{van2011lifted,van.haaren.mlj} but it has not been shown what is the computational complexity of the complete weight learning problem. 

It turns out that it is natural to study the complexity of the weight learning  problem in the relational marginal setting because one of the parameters that influences runtime is interiority of the vectors which represent marginal constraints. In particular we have the following result which provides a polynomial-time bound for maximum likelihood weight learning of 2-variable MLNs.

\begin{theorem}\label{thm:main}
Let $\Phi = \{\alpha_1,\dots,\alpha_l\}$ be a set of quantifier-free first-order logic formulas, each with at most 2 variables. Let $\Phi_0$ be a set of universally quantified first-order logic sentences, each also with at most 2 variables. Let $\Omega_{\Phi_0}$ be the set of models of $\Phi_0$ over a given domain $\Delta$. Let $\widehat{\omega} \in \Omega$ be a training example. Then there is an algorithm which finds weights of the MLN $\mathcal{M}$ given by formulas $\Phi$ such that the log-likelihood of $\mathcal{M}$ given the training example $\widehat{\omega}$ is within $\varepsilon$ of the optimum. The algorithm runs in time polynomial in $|\Delta|$, $1/\varepsilon$ and $1/\eta$ where $\eta$ is the interiority of the vector $Q_{\widehat{\omega}}(\Phi)$ in the relational marginal polytope $\textit{RMP}(\Phi,\Omega_{\Phi_0})$. 
\end{theorem}

At first, one might perhaps wonder why the above result about maximum-likelihood estimation should depend on interiority of $Q_{\widehat{\omega}}(\Phi)$. Consider the following example: $\widehat{\omega}$ represents a complete directed graph (e.g.\ using binary relations $\textit{e}/2$) and $\Phi = \{ \textit{e}(x,y) \}$. Then $Q_{\widehat{\omega}}(\Phi) = (1)$ which is clearly on the boundary of the respective polytope (in this case the polytope is just a line segment). If we try to optimize likelihood of the MLN given by $\Phi$, the weight of the formula $\textit{e}(x,y)$ will tend to infinity which also means that the optimization algorithm will not be able to converge. Thus, some dependence on interiority is necessary.

While the case from the previous paragraph might be simple to spot, there are other more tricky cases where, at first, we might not be able to realize that the weights will have to be very large. For instance, consider MLNs given by two formulas, one for edge density and one for triangle density (as in Example \ref{ex:densities}). If the training example $\widehat{\omega}$ turned out to represent a graph close to an extremal graph (see e.g.\ \cite{bollobas2004extremal}), e.g.\ one having close to maximum possible density of triangles for the given density of edges, then the learned weights would again turn out to be very large, but this time because of a more subtle reason. Again, this is what $\eta$-interiority captures.

Finally, using Theorem \ref{thm:main}, the duality of relational marginal problems and maximum-likelihood estimation in MLNs and a lemma from \cite{singh2014entropy}, we can obtain the next result about complexity of the relational marginal problems.

\begin{theorem}\label{corollary:main}
Let $\Phi$, $\Phi_0$, $\Delta$ and $\Omega_{\Phi_0}$ be as in Theorem \ref{thm:main} (in particular, all formulas in $\Phi$ and $\Phi_0$ are still required to have at most 2 variables).
Let $\eta > 0$ be a real number and $\theta = (\theta_1,\dots,\theta_l)$ be a point in the $\eta$-interior of the relational marginal polytope $\textit{RMP}(\Phi,\Omega_{\Phi_0})$. Then there exists an algorithm which finds a distribution over $\Omega_{\Phi_0}$, represented as an MLN, whose entropy is within $\varepsilon > 0$ of the maximum and which satisfies the marginal constraints $\mathbb{E}(Q_\omega(\Phi)) = \theta$ within $\sqrt{\varepsilon}$. The runtime of this algorithm is polynomial in $|\Delta|$, $1/\eta$, $1/\varepsilon$ and the number of bits needed to represent $\theta$.
\end{theorem}

\begin{remark}
We have omitted using the term ``domain-liftable'' \cite{van2011lifted} in the description of the above two results. Here is why. Suppose that we fix a vector $\theta$ and increase the domain size $|\Delta|$. It can happen that $\theta$ becomes much closer to the boundary of the polytope which means that the runtime may increase more than just polynomially with increasing $|\Delta|$ because interiority of the vector $\theta$ is one of the parameters governing the runtime. In fact, $\theta$ may end up being completely outside the polytope, rendering the problem unsolvable. One possible solution is to use interiority w.r.t.\ the polytope that we obtain as a limit for $|\Delta| \rightarrow \infty$. It follows from results in \cite{kuzelka2018relational} that polytopes over larger domains (but given by the same formulas $\Phi$) are subsets of polytopes over smaller domains (one can also obtain bounds on how much smaller the limit polytope will be compared to some polytope over a finite domain using Proposition 8 in \cite{kuzelka2018relational}). It follows that our results imply domain-liftability of the relational marginal problems for vectors $\theta$ that are in the interior of the respective limit polytopes (for $|\Delta| \rightarrow \infty$).
\end{remark}

We prove Theorem \ref{thm:main} and Theorem \ref{corollary:main} in the next sections.

\paragraph{Outline of the Proof: } First, we show how to construct relational marginal polytopes (which turn out to be needed by the algorithm) in Section \ref{sec:2var_polytopes}. Then, in Section \ref{sec:bounding_box}, following the approach from \cite{singh2014entropy} we bound the weights of the MLN which is a solution of the relational marginal problem. We finish the rest of the proofs in Sections \ref{sec:proof} and \ref{sec:proof:corollary}.

\section{POLYTOPES FOR 2-VARIABLE FORMULAS}\label{sec:2var_polytopes}

For our main result, a polynomial-time algorithm for solving relational marginal problems, we will need to be able to construct relational marginal polytopes in time polynomial in the size of the domain $\Delta$.
First, we may notice that the number of possible vectors of formulas' statistics given by a fixed set of formulas can be bounded by a polynomial in $\Delta$.

\begin{remark}
Let $\Phi = (\alpha_1,\dots,\alpha_l)$ and let $\Omega$ be a set of possible worlds over a domain $\Delta$. Let us define $\mathcal{K}(\Phi,\Omega) = \{ (Q_\omega(\alpha_1),\dots,Q_\omega(\alpha_l)) | \omega \in \Omega \}$. Then $|\mathcal{K}(\Phi,\Omega)| \leq \prod_{\alpha_i \in \Phi}(|\Delta|+1)^{|\textit{Vars}(\alpha_i)|}$, which is polynomial in $|\Delta|$.
\end{remark}

Since the relational marginal polytope $\textit{RMP}(\Phi,\Omega)$ is equal to the convex hull of $\mathcal{K}(\Phi,\Omega)$, the above remark also provides a polynomial bound for the number of its vertices.

The next proposition is a consequence of an algorithm that we describe in the appendix.

\begin{proposition}
Let $\Phi$ be a set of quantifier-free first-order logic formulas, each with at most 2 variables. Let $\Phi_0$ be a set of universally quantified first-order logic sentences, each also with at most 2 variables. Finally, let $\Omega_{\Phi_0}$ be the set of models of $\Phi_0$ over a given domain $\Delta$. Then the set of vertices of $\textit{RMP}(\Phi,\Omega_{\Phi_0})$ can be constructed in time polynomial in $|\Delta|$.
\end{proposition}

\section{BOUNDING BOX}\label{sec:bounding_box}

The main result described in this section is the following theorem which allows us to bound the magnitude of weights in MLNs that we obtain as solutions of relational marginal problems. This theorem is a relational counterpart of Theorem 2.7 from \cite{singh2014entropy}. The proof follows the steps of the respective proof from \cite{singh2014entropy} and most of the heavy-lifting has already been done there (however, we do need to generalize their results to our setting).

\begin{theorem}\label{thm:bounding_box}
Let $\Phi$ be a set of quantifier-free first-order logic formulas, let $\Omega$ be a set of possible worlds and $A^= \mathbf{x} = \mathbf{c}$ be a maximal system of linearly independent equations satisfied by the vertices of the relational marginal polytope $\mathbf{P}_R = \textit{RMP}(\Phi,\Omega)$. Let $\theta$ be a point in the $\eta$-interior of $\mathbf{P}_R$. Then there is an optimal solution  $\mathbf{\lambda}^*$ of the dual problem (\ref{eq:dual}) such that $A^= \mathbf{\lambda}^* = 0$ and any such solution satisfies $\|\mathbf{\lambda}^*\| \leq \log{|\Omega|}/\eta$. 
\end{theorem}

To prove this theorem we start with some lemmas.
In what follows, when $\Phi = (\alpha_1,\dots,\alpha_l)$ is a list of formulas, we will use the notation $Q_\omega(\Phi) \triangleq (Q_\omega(\alpha_1),\dots,Q_\omega(\alpha_l))$.

\begin{lemma}\footnote{This is a relational counterpart of Lemma 5.1 from \cite{singh2014entropy}.}\label{lemma:omegabound}
Let $\Phi = (\alpha_1,\dots,\alpha_l)$, $\mathbf{\theta} = (\theta_1,\dots,\theta_l)$ be a point in the $\eta$-interior of the relational marginal polytope $\mathbf{P}_R = \textit{RMP}(\Phi,\Omega)$ and let $\mathbf{\lambda}^* = (\lambda_1^*,\dots,\lambda_l^*)$ be the optimal solution to the dual problem (\ref{eq:dual}). Then for any $\mathbf{x} \in \mathbf{P}_R$: $\langle \mathbf{\lambda}^*,\mathbf{x}-\mathbf{\theta} \rangle \leq \log{|\Omega|}$.
\end{lemma}
\begin{proof}
The entropy of any distribution which is a solution of the relational marginal problem is bounded by $\log{|\Omega|}$, which is the entropy of the uniform distribution over $\Omega$. It follows from strong duality that $-L(\mathbf{\lambda}^*) \leq \log{|\Omega|}$ where $L(\mathbf{\lambda}^*)$ is defined in (\ref{eq:dual}). Hence 
\begin{multline*}
- L(\mathbf{\lambda}^*) = -\langle \mathbf{\lambda}^*, \mathbf{\theta} \rangle + \log{\sum_{\omega \in \Omega} e^{\langle \mathbf{\lambda}^*, Q_\omega(\Phi) \rangle}} \leq \log{|\Omega|}.
\end{multline*}
In particular, for every $\omega \in \Omega$: 
\begin{equation}\label{eq:lemma:aux1}
-\langle \mathbf{\lambda}^*, \mathbf{\theta} \rangle + \langle \mathbf{\lambda}^*, Q_\omega(\Phi) \rangle \leq \log{|\Omega|}.
\end{equation}
Since $x \in \mathbf{P}_R$, we can write it as a convex combination $x = \sum_{\omega \in \Omega} a_\omega \cdot Q_{\omega}(\Phi) $. Using (\ref{eq:lemma:aux1}) we obtain
$$\sum_{\omega \in \Omega}(- a_\omega \langle \mathbf{\lambda}^*, \mathbf{\theta} \rangle + a_\omega \langle \mathbf{\lambda}^*, Q_\omega(\Phi) \rangle) \leq \sum_{\omega \in \Omega} a_\omega \log{|\Omega|}. $$
Since $\sum_{\omega \in \Omega} a_\omega = 1$ (recall that we represented $\mathbf{x}$ as a {\em convex combination}), we obtain:
$\langle \mathbf{\lambda}^*,\mathbf{x}-\mathbf{\theta} \rangle \leq \log{|\Omega|}$.
\end{proof}

\begin{lemma}\footnote{This is a relational counterpart of Lemma 2.5 in \cite{singh2014entropy}.}\label{lemma:nulllam}
Let $A^= \mathbf{x} = \mathbf{c}$ be a maximal linearly-independent system of linear equations which are satisfied by all vertices of the relational marginal polytope $\mathbf{P}_R = \textit{RMP}(\Phi,\Omega)$. Then, for any $\mathbf{d} \in \mathbb{R}^{m}$ where $m$ is the column dimension of $A^=$, $L(\mathbf{\lambda}) = L(\mathbf{\lambda} + (A^=)^T \mathbf{d})$ where $L$ is as in (\ref{eq:dual}).
\end{lemma}
\begin{proof}
First, for any $\omega \in \Omega$: $A^= Q_\omega(\Phi) = c$. Second we can write $\mathbf{\theta} = \sum_{\omega \in \Omega} a_\omega Q_\omega(\Phi)$, where $\sum_{\omega \in \Omega} a_\omega = 1$.

Next, we have
\begin{multline*}
    \langle \mathbf{\lambda}+(A^=)^T \mathbf{d}, \mathbf{\theta} \rangle = \langle \mathbf{\lambda}, \mathbf{\theta} \rangle + \langle (A^=)^T \mathbf{d}, \mathbf{\theta} \rangle \\
    = \langle \mathbf{\lambda}, \mathbf{\theta} \rangle + \sum_{\omega \in \Omega} a_\omega \langle (A^=)^T \mathbf{d}, Q_\omega(\Phi) \rangle \\
    = \langle \mathbf{\lambda}, \mathbf{\theta} \rangle + \sum_{\omega \in \Omega} a_\omega \langle \mathbf{d}, A^= Q_\omega(\Phi) \rangle = \langle \mathbf{\lambda}, \mathbf{\theta} \rangle + \langle \mathbf{d}, \mathbf{c} \rangle.
\end{multline*}
For the dual problem (\ref{eq:dual}), we have
\begin{multline*}
    L(\mathbf{\lambda} + (A^=)^T \mathbf{d}) = \langle \mathbf{\lambda}+(A^=)^T \mathbf{d}, \mathbf{\theta} \rangle \\
    - \log{\sum_{\omega \in \Omega} e^{\langle \mathbf{\lambda}+(A^=)^T \mathbf{d}, Q_\omega(\Phi) \rangle}} = \langle \mathbf{\lambda}, \mathbf{\theta} \rangle + \langle \mathbf{d}, \mathbf{c} \rangle \\
    - \log{\sum_{\omega \in \Omega} e^{\langle \mathbf{d}, \mathbf{c} \rangle+\langle \mathbf{\lambda}, Q_\omega(\Phi) \rangle}} \\
    = \langle \mathbf{\lambda}, \mathbf{\theta} \rangle - \log{\sum_{\omega \in \Omega} e^{\langle \mathbf{\lambda}, Q_\omega(\Phi) \rangle}} = L(\mathbf{\lambda}).
\end{multline*}
\end{proof}

Due to the above lemma and since $A^=$ represents the maximal set of linearly independent equalities satisfied by points of $\mathbf{P}_R$, we can restrict ourselves to $\mathbf{\lambda}$'s that satisfy $A^= \mathbf{\lambda} = 0$ in the search for the optimal solution of the dual problem (\ref{eq:dual}).

The next lemma, which we will also need for the proof of Theorem \ref{thm:bounding_box}, does not need to be adapted and can be used for our purposes as is; we refer to \cite{singh2014entropy} for proof.

\begin{lemma}[Lemma 5.2 in \cite{singh2014entropy}]\label{lemma:qqb}
Let $A^= \mathbf{x} = \mathbf{c}$ be a system of linear equations, $\mathbf{\theta} \in \mathbb{R}^m$ and $\eta \geq 0$. Let us define three sets $\mathcal{B}$, $\mathcal{Q}$ and $\tilde{\mathcal{Q}}$:
\begin{align*}
    \mathcal{B}(\mathbf{\theta}) &= \{ \mathbf{x} \in \mathbb{R}^m | A^= \mathbf{x} = \mathbf{c}, \|\mathbf{x} - \mathbf{\theta} \| \leq \eta \}, \\
    \mathcal{Q}(\mathbf{\theta}) &= \{ \mathbf{y} \in \mathbb{R}^m | A^= \mathbf{y} = \mathbf{c}, \| \mathbf{y} - \mathbf{\theta} \| \leq 1/\eta \}, \\
    \tilde{\mathcal{Q}}(\mathbf{\theta}) &= \{ \mathbf{z} \in \mathbb{R}^m | A^= \mathbf{z} = \mathbf{c}, \forall x \in \mathcal{B}(\mathbf{\theta}) : \\
    & \quad\quad\quad\quad\quad\quad\quad\quad\quad \langle \mathbf{z}-\mathbf{\theta}, \mathbf{x}-\mathbf{\theta} \rangle \leq 1 \}.
\end{align*}
Then $\mathcal{Q} = \tilde{\mathcal{Q}}$.
\end{lemma}

We are now ready to prove Theorem \ref{thm:bounding_box}.

\begin{proof}[Proof of Theorem \ref{thm:bounding_box}]
Let $\mathbf{\lambda}^*$ be an optimal solution of the dual problem (\ref{eq:dual}) satisfying $A^= \mathbf{\lambda}^* = 0$. This can be chosen because of Lemma~\ref{lemma:nulllam}.
Let $\mathcal{Q}(\mathbf{\theta})$, $\tilde{\mathcal{Q}}(\mathbf{\theta})$ and $\mathcal{B}(\mathbf{\theta})$ be as in Lemma \ref{lemma:qqb}.
Let us define 
$$\tilde{\lambda} = \frac{\mathbf{\lambda}^*}{\log{|\Omega|}} + \mathbf{\theta}.$$
We will first show that $\tilde{\lambda} \in \tilde{\mathcal{Q}}(\mathbf{\theta})$. We have 
$$A^= \tilde{\lambda} = A^=  \frac{\mathbf{\lambda}^*}{\log{|\Omega|}} + A^= \mathbf{\theta} = A^= \theta = c.$$
Thus, $\tilde{\lambda} \in \mathbf{P}_R$. Let $x \in \mathcal{B}$. Then we have
$$\langle \tilde{\mathbf{\lambda}} - \mathbf{\theta}, \mathbf{x} - \mathbf{\theta} \rangle = \frac{\langle \mathbf{\lambda}^*, \mathbf{x} - \mathbf{\theta} \rangle}{\log{|\Omega|}} \leq \frac{\log{|\Omega|}}{\log{|\Omega|}} = 1$$
where the inequality follows from Lemma \ref{lemma:omegabound}. Thus $\tilde{\mathbf{\lambda}} \in \tilde{\mathcal{Q}}(\mathbf{\theta}) = \mathcal{Q}(\mathbf{\theta})$ by Lemma~\ref{lemma:qqb}. From the definition of $\mathcal{Q}(\mathbf{\theta})$, we have
$$1/\eta \geq \| \tilde{\mathbf{\lambda}} - \mathbf{\theta} \| = \left\| \frac{\mathbf{\lambda}^*}{\log{|\Omega|}} \right\|.$$
It follows that $\| \mathbf{\lambda}^* \| \leq \log{|\Omega|}/\eta$, finishing the proof.

\end{proof}

\section{PROOF OF THEOREM \ref{thm:main}}\label{sec:proof}

In this section we prove Theorem \ref{thm:main} by showing how to solve the dual problem (\ref{eq:dual}) using the ellipsoid algorithm. 

First, in order to run the ellipsoid algorithm, we need a {\em first-order oracle}, i.e.\ we need a procedure to compute $L(\lambda)$ and $\nabla L(\lambda)$. This can be computed by WFOMC using the encoding from Section \ref{sec:wfomc}. In particular, as discussed in Section \ref{sec:wfomc}, when both $\Phi$ and $\Phi_0$ contain formulas with at most 2 variables, we can compute WFOMC in time polynomial in the size of the domain $|\Delta|$. Hence, in this case we will have a first-order oracle running in time polynomial in $|\Delta|$.

Second, since we have to search for solutions $\lambda^*$ satisfying $A^= \lambda^* = 0$, where the matrix $A^=$ is defined as in Section \ref{sec:bounding_box}, we need to be able to compute $A^=$. For the case when both $\Phi$ and $\Phi_0$ contain formulas with at most 2 variables, we can compute the set of vertices of the relational marginal polytope in time polynomial in $|\Delta|$ as discussed in Section \ref{sec:2var_polytopes}. Finding the matrix $A^=$ is then a straightforward linear algebraic problem. One can then show, using the fact that the number of vertices of the relational marginal polytope is polynomial in $|\Delta|$ and that the representation of these vertices is polynomial in $|\Delta|$ as well, that the number of bits needed to encode $A^=$ and $\mathbf{c}$ is also polynomial in $|\Delta|$. 

Since we have a first-order oracle and we also have means to compute the matrix $A^=$ and the vector $\mathbf{c}$ which together represent the constraints, we can run the ellipsoid algorithm. However, what remains to be shown is how long the ellipsoid algorithm will need to run in order to obtain a solution with value that is no more than $\varepsilon$ from the optimum. We do that next.

Using Theorem \ref{thm:bounding_box} and Theorem \ref{th:ellipsoid}, if we set $R = \log{|\Omega|}/\eta$ and
$$
\beta = - \frac{\varepsilon}{ \left(\min_{\lambda\in K,\|\lambda\|_\infty\le R} L(\lambda)-\max_{\lambda\in K,\|\lambda\|_\infty\le R} L(\lambda)\right)}    
$$
then the ellipsoid algorithm will find a solution of the dual problem (\ref{eq:dual}) with value within $\varepsilon$ from the optimum in time polynomial in $\log{R}$, $l$ and $\log{(1/\beta)}$.

Hence we need to bound $\beta$. First, since $L(\lambda) \leq 0$, we can just focus on bounding $\min_{\lambda\in K,\|\lambda\|_\infty\le R} L(\lambda)$. We have
\begin{multline*}
    -L(\lambda) = -\langle \lambda,\theta \rangle + \log{\sum_{\omega \in \Omega} e^{\langle \mathbf{\lambda}, Q_\omega(\Phi) \rangle}} \\
    \leq |\langle \lambda,\theta \rangle| + \left|\log{\sum_{\omega \in \Omega} e^{\langle \mathbf{\lambda}, Q_\omega(\Phi) \rangle}}\right|
    \leq l \frac{\log{|\Omega|}}{\eta} \\ + \log{\left( |\Omega| \cdot \exp{\left( l \frac{\log{|\Omega|}}{\eta} \right)} \right)} 
    \leq (2l+1) \frac{\log{|\Omega|}}{\eta}.
\end{multline*}
Hence, $L(\omega) \geq -(2l+1) \frac{\log{|\Omega|}}{\eta}$ and 
$\beta \geq \frac{\varepsilon \eta}{ (2l+1) \log{|\Omega|}}.$
It follows that the number of WFOMC calls which the ellipsoid algorithm needs to run is polynomial in $\log{(\log{|\Omega|}/\eta)}$, $\log{( (2l+1) \log{|\Omega|}/ (\varepsilon \eta))}$ and $l$. Finally, noting that each of these calls can be performed in time polynomial in $|\Delta|^c$ and $\log{|\Omega|}/\eta$ (recall that $\log{|\Omega|}/\eta$ defines the bounding box where we need to search) and that $\log{|\Omega|} = O(|\Delta|^{c'})$ finishes the proof (here the constant $c$ depends on $\Phi$ and $\Phi_0$ and the constant $c'$ depends on the given first-order language~$\mathcal{L}$). \qed


\section{PROOF OF THEOREM \ref{corollary:main}}\label{sec:proof:corollary}

Here we prove Theorem \ref{corollary:main}. For that we also need the following lemma, which is just a reformulation of Lemma A.4 from \cite{singh2014entropy} using our notation.

\begin{lemma}\label{lemma:vardist}
Let $\lambda^*$ be an optimal solution of the dual problem (\ref{eq:dual}) and let $\lambda$ be such that $L(\lambda) \geq L(\lambda^*)-\varepsilon$. Then 
$$L(\lambda^*)-L(\lambda) = D_{KL}(p^*||p) \leq \varepsilon$$
where $p^*$ is the MLN given by the formulas from $\Phi$ with weights $\lambda^*$ and $p$ is the MLN given by the same formulas $\Phi$ with weights $\lambda$.
\end{lemma}

Next from Pinsker's inequality we have $\delta_{TV}(p^*,p) \leq \sqrt{D_{KL}(p^*||p)}$ where $\delta_{TV}(p^*,p)$ denotes the total variation distance of $p^*$ and $p$ and $p$ and $p^*$ are as in Lemma \ref{lemma:vardist}. Finally, realizing that $|\mathbb{E}_{\omega \sim p^*}[Q_\omega(\Phi)] - \mathbb{E}_{\omega \sim p}[Q_\omega(\Phi)]| \leq \delta_{TV}(p^*,p)$ together with the result in Theorem \ref{thm:main} and with the duality finishes the proof of Theorem \ref{corollary:main}. \qed

\section{CONCLUSIONS}\label{sec:conclusions}

We have proved that maximum-likelihood weight learning of MLNs given by formulas with at most 2 variables can be solved in time polynomial in the size of the domain $\Delta$. In order to obtain this result, we framed the learning problem as a relational marginal problem which allowed us to exploit algorithmic techniques from \cite{singh2014entropy}. Some of the new results that we obtained in this paper hold for general MLNs, not just the 2-variable ones. For instance, Theorem \ref{thm:bounding_box} holds for all MLNs. The bounds on the number of steps of the ellipsoid algorithm following from the results in Sections \ref{sec:proof} and \ref{sec:proof:corollary} hold for general MLNs as well. We believe that not only the result but also the techniques could be useful for SRL.

We should also stress here that the algorithm described in this paper is meant mostly for theoretical purposes; it is not the most practical one. A more practical algorithm could be obtained if we replaced the ellipsoid algorithm by the projected gradient descent algorithm and designed a more practical variant of the algorithm for construction of relational marginal polytopes. 

\noindent {\bf Acknowledgments} A significant part of this work was done while OK was with KU Leuven, supported by Research Foundation - Flanders (project G.0428.15). OK and VK were supported by the OP VVV project {\it CZ.02.1.01/0.0/0.0/16\_019/0000765} ``Research Center for Informatics''.

\fi 

\ifappendix
\appendix

\section{COMPUTING POLYTOPES FOR 2-VARIABLE FORMULAS}

In this section we describe an algorithm for constructing relational marginal polytopes given by sets of first-order formulas, each with at most 2 logical variables. 
The algorithm described in this section is largely inspired by the WFOMC algorithm from \cite{beame2015symmetric}. 
In what follows in this section, we will denote by $\Omega_{\Phi_0}$ the set of possible worlds over domain $\Delta= \{c_1,\dots, c_{|\Delta|} \}$ which satisfy a given set $\Phi_0$ of universally quantified first-order logic sentences.\footnote{Existential quantifiers can be treated using a form of Skolemization we omit the details here.}

We need an algorithm which can compute the set $\mathcal{K}(\Phi,\Omega_{\Phi_0})$ defined in Section 5. 
Let $\mathcal{U}$ be the set of all unary predicates in the considered first-order language $\mathcal{L}$ and $\mathcal{B}$ be the set of all binary predicates (for 2-variable formulas, we may assume w.l.o.g.\footnote{We refer to \cite{beame2015symmetric} for details.}\ that $\mathcal{L}$ does not contain any literals of arity higher than $2$). In the following, we will use the notion of {\em cells}, which was also used in \cite{beame2015symmetric}. Given a possible world $\omega$, we say that two constants $c, c' \in \Delta$ are in the same {\em cell} if for all $u \in \mathcal{U}$ we have $\omega \models u(c)$ iff $\omega \models u(c')$; each cell can then be identified by a subset of $\mathcal{U}$ naturally.


\begin{remark}\label{remark:omegar}
Suppose that $\mathcal{B} = \emptyset$ (i.e.\ we only have unary predicates) and that $\Phi_0$ and $\Phi$ are constant-free. Then we can construct the set $\mathcal{K}(\Phi,\Omega_{\Phi_0})$ in polynomial time as follows. First, we construct an auxiliary set of all integer partitions of $|\Delta|$:
$$\mathcal{J} = \left\{ (j_1,\dots,j_{|2^{\mathcal{U}}|}) \left| \sum_{k=1}^{|2^{\mathcal{U}}|} j_k = |\Delta| \wedge \forall k : j_k \geq 0 \right. \right\} $$
The intention is that the $i$-th entry of a vector $J \in \mathcal{J}$ should represent the number of constants $c \in \Delta$ that are in the $i$-th cell (here the cells will be ordered arbitrarily in some order).
We can then use the set $\mathcal{J}$ to define a set of possible worlds $\Omega_R \subseteq \Omega_{\Phi_0}$ which will be representative of all the possible worlds in the sense that $\mathcal{K}(\Phi,\Omega_{\Phi_0}) = \{ (Q_\omega(\alpha_1),\dots,Q_\omega(\alpha_l)) | \omega \in \Omega_R \}$. We define the set $\Omega_R$ as follows. First we order (arbitrarily) the constants in $\Delta$ and we do the same with the sets in $2^\mathcal{U}$; we denote by $c_i$ the $i$-th constant and similarly, by $U_i$, the $i$-th subset of $\mathcal{U}$.
For every $J = (j_1, \dots, j_{|2^{\mathcal{U}}|}) \in \mathcal{J}$ we construct:
\begin{multline*}
  \omega_J = \bigcup_{i = 1}^{j_1} \bigcup_{R \in U_1} \left\{ R(c_{i}) \right\} \cup \bigcup_{i = j_1+1}^{j_1+j_2} \bigcup_{R \in U_2} \left\{ R(c_{i}) \right\} \cup \dots \\
  \dots \cup \bigcup_{i = j_1 + \dots + j_{\left|2^\mathcal{U}\right|-1} +1}^{|\Delta|} \bigcup_{R \in U_{\left|2^\mathcal{U}\right|}} \left\{ R(c_{i}) \right\}
\end{multline*}
Then we define $\Omega_R = \{ \omega_J | J \in \mathcal{J} \}$. Notice that $|\Omega_R|$ is polynomial in $|\Delta|$. Finally, it is easy to show that we can do the following in polynomial time (i.e.\ polynomial in $|\Delta|$): (i) to filter out possible worlds that do not satisfy $\Phi_0$ and (ii) to compute $(Q_{\omega}(\alpha_1),\dots,Q_\omega(\alpha_l))$.
\end{remark}

In the next example we illustrate the construction from the above remark.

\begin{example}
Let $\mathcal{U} = \{ \textit{sm}/1 \}$ and $\Delta = \{ \textit{Alice}, \textit{Bob} \}$. Then $\mathcal{J} = \{ (0,2), (1,1), (2,0) \}$. Now, for every $J \in \mathcal{J}$, we need to construct the respective $\omega_J$. That is, for the ordering of constants $\textit{Alice} \prec \textit{Bob}$ and the ordering of cells $\emptyset \prec \{ \textit{sm}/1 \}$, we have:
\begin{align*}
    \omega_{(0,2)} &= \{ \textit{sm}(\textit{Alice}), \textit{sm}(\textit{Bob}) \}, \\
    \omega_{(1,1)} &= \{ \textit{sm}(\textit{Bob}) \}, \\
    \omega_{(2,0)} &= \emptyset.
\end{align*}
The set of representative possible worlds is $\Omega_R = \{ \omega_{(0,2)}, \omega_{(1,1)}, \omega_{(2,0)} \}$.
\end{example}

We now need to explain how to compute the set $\mathcal{K}(\Phi, \Omega_{\Phi_0})$ for the case when $\mathcal{B} \neq \emptyset$. 
We again show how to construct the set of representative possible worlds but this time also with binary predicates; we denote this set $\Omega_R^B$. We will explain how to construct representatives by extending one possible world $\omega_0 \in \Omega_R$, constructed as in Remark \ref{remark:omegar}. Hence, obviously the same procedure will need to be repeated for all possible worlds from $\Omega_R$.

\begin{remark}\label{remark:simplebinary}
First, we consider literals of the form $R(c,c)$ where $R \in \mathcal{B}$ and $c \in \Delta$. We can notice that these literals can be added already in the construction of $\Omega_R$ (using auxiliary unary predicates), so we will not consider this type of literals here further. 
\end{remark}

The next remark will provide us with a simple way to construct the set of representatives. 

\begin{remark}\label{remark:binary}
Let us suppose that the possible world $\omega_J$, where $J = (j_1,\dots,j_{\left|2^\mathcal{J}\right|}) \in 2^\mathcal{J}$, is as in Remark \ref{remark:omegar}. We first discuss how we could generate all possible worlds that could be obtained from $\omega_J$. Let $\Delta_q = \{ c_{\sum_{k = 1}^{q-1} j_k + 1}, \dots, c_{\sum_{k = 1}^{q} j_k} \},$ and
$\Delta_r = \{ c_{\sum_{k = 1}^{r-1} j_k + 1}, \dots, c_{\sum_{k = 1}^{r} j_k} \}.$ 
Next we could assign a subset of binary predicates $\mathcal{B}$ to each element of the set $\{ (c,c') \in (\Delta_q \times \Delta_r) | c \neq c' \}$ (note that the condition $c \neq c'$ is only relevant for $r = q$ and note that we have already taken care of literals of the form $R(c,c)$). If for instance, $(c_1,c_2)$ got assigned the predicates $\textit{friends}$, $\textit{teammates}$ then we would include the literals $\textit{friends}(c_1,c_2)$ and $\textit{teammates}(c_1,c_2)$ to the constructed possible world, and analogically for all the other tuples. 
Finally, let us define $\#(B, q, r)$ to be the number of pairs of domain elements from $\Delta_q \times \Delta_r$ which are assigned the subset of binary predicates $B \in 2^\mathcal{B}$. We may notice that $Q_{\omega}(\alpha)$ for any 2-variable  quantifier-free formula $\alpha$ will only depend on the numbers $\#_\omega(B, q, r)$ but not on any other details of the possible worlds. The same also holds for the 2-variable universally quantified formulas in $\Phi_0$. Hence, we can construct only representatives with distinct $\#_\omega(B, q, r)$'s using a straightforward generalization of the procedure from Remark \ref{remark:omegar}.
\end{remark}

Finally, we need to show that the number of representatives in the set constructed according to Remark \ref{remark:binary} has size polynomial in $|\Delta|$. Using Remarks \ref{remark:omegar}, \ref{remark:simplebinary} and \ref{remark:binary}, we can obtain the rather crude upper bound:
$$|\Omega_R^B| \leq (|\Delta|+1)^{2^{|\mathcal{U}|+|\mathcal{B}|}} \cdot {(|\Delta|+1)^{2 \cdot 4^{|\mathcal{U}|+|\mathcal{B}|} \cdot 2^{|\mathcal{B}|}}}.$$
Here, the first part comes from Remarks \ref{remark:omegar} and \ref{remark:simplebinary} and the second part from Remark \ref{remark:binary}. Importantly, the bound is polynomial in $|\Delta|$. 
Since our main aim in this paper is establishing existence of polynomial-time algorithms for weight learning, we will not try to optimize this bound. In practice, one could probably find the vertices defining the polytope faster using a generic SAT solver as an oracle inside a heuristic algorithm iteratively traversing vertices of the polytope, but that would not lead to an algorithm with runtime polynomial in the size of the domain.

\fi

\bibliographystyle{abbrv}
\bibliography{ref}

\end{document}